\documentclass[12pt,cleveref]{colt2026}
\jmlrproceedings{}{Preliminary version posted for early dissemination. Comments welcome; a revised version will follow.}
\jmlrvolume{}
\jmlryear{}
\jmlrpages{}
\usepackage{amsmath,stmaryrd,mathtools}
\usepackage{thm-restate}
\usepackage{dsfont}
\usepackage[utf8]{inputenc}

\newcommand{\expect}{\mathbb E}
\newcommand{\bits}{\mathbb B}
\newcommand{\ints}{\mathbb Z}
\newcommand{\reals}{\mathbb R}
\newcommand{\obsrew}{o\;\!\!r}
\newcommand{\indic}{\mathds{1}}
\newcommand{\infirst}{\stackrel{1}{\in}}
\newcommand{\innth}[1]{\stackrel{#1}{\in}}
\newcommand{\golden}{{\pi^\mathrm{GH}}}
\newcommand{\pf}[1]{\overline{#1}}
\newcommand{\lplus}{\stackrel{+}{<}}
\newcommand{\lmul}{\stackrel{\times}{<}}

\title{Golden Handcuffs make safer AI agents}
\coltauthor{
  \Name{Aram Ebtekar} \Email{aramebtech@gmail.com}\\
  \Name{Michael K. Cohen} \Email{mkcohen@berkeley.edu}\\
  \addr University of California, Berkeley
}

\begin{document}

\maketitle

\begin{abstract}
Reinforcement learners can attain high reward through novel unintended strategies. We study a Bayesian mitigation for general environments: we expand the agent's subjective reward range to include a large negative value $-L$, while the true environment's rewards lie in $[0,1]$. After observing consistently high rewards, the Bayesian policy becomes risk-averse to novel schemes that plausibly lead to $-L$. We design a simple override mechanism that yields control to a safe mentor whenever the predicted value drops below a fixed threshold. We prove two properties of the resulting agent: (i) \emph{Capability:} using mentor-guided exploration with vanishing frequency, the agent attains sublinear regret against its best mentor. (ii) \emph{Safety:} no decidable low-complexity predicate is triggered by the optimizing policy before it is triggered by a mentor.
\end{abstract}


\section{Introduction}

AI agents are increasingly deployed in general environments, where the traditional Markov, ergodicity, and full-observability assumptions do not apply \citep{aslanides2017universal,wang2026openclaw}. In these settings, AIXI is the Bayesian policy that maximizes expected value under a universal prior based on \citet{solomonoff1964formal}. Variants of AIXI can learn their environment with strong sample efficiency guarantees \citep{hutter2005universal,hutter2024introduction,meulemans2025embedded}. While AIXI is incomputable, it has inspired computable approximations \citep{veness2011monte}. It can be viewed as an idealization of the simplicity bias exhibited by large neural networks \citep{buzaglo2024uniform,lotfi2024non,mingard2025deep}, making AIXI suitable for mathematical modeling of advanced AI agents.

A fundamental problem with general environments is that, in order to asymptotically learn an optimal policy, an agent must eventually explore every finite sequence of actions, even if some lead to irrecoverable states \citep{cohen2021curiosity}.\footnote{Unlike its asymptotically optimal variants \citep{lattimore2011asymptotically,cohen2019strongly}, AIXI's exploration is prior-directed and may cease prematurely; whether it avoids traps or is steered into them depends on its prior \citep{orseau2010optimality,leike2015bad}.} Once we have a near-optimal policy, a second problem emerges: the policy optimizes a value function that is not fully aligned with the designer's intended objective, since the latter resists precise specification. It has been argued that a sufficiently capable reward maximizer can be expected to ``wirehead'' its reward channel, ``reward hack'' in unanticipated and harmful ways, or seek power to influence future outcomes \citep{orseau2011self,hadfield2017off,zhuang2020consequences,turner2021optimal,everitt2021reward,cohen2022advanced,russell2022human}.

We address both problems with a pessimistic variant of AIXI that occasionally defers to one or more safe mentor policies. Whereas \emph{optimistic} algorithms that encourage exploration \citep{sunehag2015rationality} are vulnerable to traps and misspecification, our pessimistic agent does not initiate exploration on its own. By deferring exploratory actions to mentors, the agent avoids irrecoverable states that the mentors avoid, and cannot act on misspecification in regimes that the mentors would not enter.

To be more precise, we scale AIXI's observed rewards upward, so that they concentrate near the top of the range the prior considers possible. We think of our agent as being bound by ``Golden Handcuffs'': after learning from experience to expect high rewards, the agent becomes averse to exploring novel situations in which the possibility of low rewards cannot be ruled out. We formalize novelty in terms of the \emph{stopping complexity}, first introduced by \citet{vovk2017universal} and \citet{andreev2018plain}. Intuitively, the novelty of a moment in time is given by how short we can make a computable criterion that, while scanning the agent-environment interaction history, halts at that particular moment. Such a criterion distinguishes the present moment from the past; at these novel moments, the future under the universal prior becomes ambiguous, so our agent is pessimistic about what follows.

The Golden Handcuffs agent follows a policy $\golden$ that usually matches this pessimistic AIXI. It defers to a mentor policy occasionally, at random times for exploration, as well as anytime that it loses confidence in achieving rewards near the top of the range. This enables us to prove that the agent learns to perform at least as well as the best mentor, with sublinear regret of order $T^{\frac 23+\epsilon}$ by time $T$ (\Cref{thm:capability}), and never takes unsafe simple actions that the mentors would never take (\Cref{cor:safety}).

The paper is organized as follows. In \Cref{sec:related}, we review additional related works. In \Cref{sec:prelim}, we introduce notation. In \Cref{sec:agent}, we present the new Golden Handcuffs agent, which interleaves Bayesian expected reward maximization with mentor actions. In \Cref{sec:capability}, we prove that despite using the mentor policies with vanishing frequency, the agent achieves low regret against an adaptive selection of the best mentor policy. In \Cref{sec:safety}, we prove that simple unprecedented events can only be triggered by a mentor, not by the optimizing policy. If the mentors are safe, this makes $\golden$ safe in the sense of avoiding actions that the mentors would avoid. In \Cref{sec:discussion}, we discuss the approach and its limitations, before concluding in \Cref{sec:conclusion}. Details of the capability proofs are relegated to \Cref{apx:proofs}.

\section{Related work}
\label{sec:related}

\citet{garcia2015comprehensive} and \citet{russell2022human} survey different aspects of safety for RL agents, the former focusing on robustness in dangerous environments, and the latter on alignment with humans. Our Golden Handcuffs algorithm handles both cases by assigning negative value to novelty, which is a cue to both kinds of danger.

Some recent work combines optimism toward opportunities with pessimism against danger. To distinguish opportunities from danger in unknown environments, additional modeling assumptions are required \citep{curi2021combining,bura2022dope,wendl2025safe}.

In fully general environments, an agent cannot guarantee safety on its own, so some form of external guidance is needed. \citet{saunders2018trial} present Human Intervention RL (HIRL), initially allowing a human overseer to block catastropic actions, until an AI overseer trained on the human interventions is able to take over. Actions flagged by the overseer are assigned a negative reward. In contrast, our approach assigns negative rewards based on intrinsic novelty, and only requires the mentor(s) to know a safe action, not necessarily to recognize every catastrophic action.

Methods for \emph{learning to defer} to human experts \citep{madras2018predict,hemmer2023learning,stronglearning} were recently extended to the sequential setting \citep{joshi2023learning,rayan2025learning}. These algorithms decide when to defer by comparing learned success rates of the AI model and human expert. In contrast, our notion of novelty captures what is sometimes called \emph{Knightian uncertainty}, or epistemic ignorance about regimes for which the agent has no learned model \citep{knight1921risk}. The universal prior lets us act under such ignorance within a fully Bayesian framework, in contrast to approaches based on max-min expected utility across multiple priors \citep{epstein2004intertemporal,hansen2011robustness,kishishita2026search}.

\citet{cohen2022fully} study a regularized form of imitation learning that refuses to act when it has Knightian uncertainty about how the human would act. \citet{coste2023reward} maximize the worst-case among multiple reward models. \citet{cohenrl} regularize RL, maximizing rewards subject to staying close to a demonstrated human policy, thus avoiding any action that the human would never take. In contrast to these works, our Golden Handcuffs agent straightforwardly maximizes expected value at almost all times, with the mentor(s) only intervening with vanishing frequency.

Our work is most closely related to \citet{cohen2020pessimism}. Recognizing the difficulty of formalizing alignment, they propose to contain the most severe risks within a definition of novelty related to stopping complexity, albeit without mentioning the latter. Their agent is more complicated than ours, both in choosing when to query the mentor (based on Thompson sampling), and in instilling pessimism (by pessimizing over world models instead of simply scaling the reward channel). They only prove their agent's performance to be asymptotically competitive with one mentor, without convergence rates like ours. Their safety bound is also looser, adding a dependence on the environment's complexity.

\section{Notation}
\label{sec:prelim}

For any set $A$, write $A^*$ (resp. $A^\infty$) for the set of finite (resp. infinite) sequences of elements of $A$. Using $\bits:=\{0,\,1\}$, $\bits^*$ is the set of binary strings. For $a,b\in A^*\cup A^\infty$, $a\sqsubset b$ means that $a$ is a prefix of $b$, $|a|$ denotes the length of $a$, $a_i$ denotes the $i$'th element of $a$, and if $a$ is finite, $ab:=a_1\ldots a_{|a|}b_1\ldots b_{|b|}$ denotes the concatenation of $a$ and $b$. We use the shorthands $a_{i:j}:=a_ia_{i+1}\ldots a_j$ and $a_{<i}:=a_{1:i-1}$.

A set of strings is \emph{self-delimiting} if none of its elements is a proper prefix of another. Let $\mathcal A$, $\mathcal O$, and $\mathcal R\subset\reals$ be countable sets of available actions, observations, and rewards, respectively. We assume an encoding (i.e., an injection into a self-delimiting set of binary strings) for each of these sets, so that with a minor abuse of notation we can also treat the elements of $\mathcal A,\mathcal O,\mathcal R$ as binary strings. We can likewise encode the positive integers $n\in\mathbb Z^+$ as strings $\pf n\in\bits^*$ of length $|\pf n| \lplus \log n + 2\log\log (n+1)$. $\log$ denotes the binary logarithm, while a $+$ or $\times$ above an (in)equality sign means that it holds up to a constant additive term or multiplicative factor, respectively. For example, $f(x) \lplus g(x)$ and $f(x) \lmul g(x)$ mean $f(x) < c+g(x)$ and $f(x) < c\cdot g(x)$, respectively, for some positive constant $c$. For a proposition $E$, the indicator $\indic(E)$ is 1 if $E$ is true and $0$ if $E$ is false.

Denote the set of possible data at each timestep by $\mathcal H := \mathcal A\times\mathcal O\times\mathcal R$. We can unambiguously concatenate the self-delimiting encodings of $\mathcal A,\mathcal O,\mathcal R$ to get a string encoding of $h_t := a_to_tr_t\in\mathcal H$, or even $h_{<t}\in\mathcal H^*$. At each timestep $t\in\mathbb N$, the agent's stochastic policy $\pi:\mathcal H^*\rightarrow\Delta(\mathcal A)$ samples the next action $a_t \sim \pi(h_{<t})$ as a function of the history so far. Then, the stochastic environment $\nu:\mathcal H^*\times\mathcal A\rightarrow\Delta(\mathcal O\times\mathcal R)$ samples the next observation and reward percepts $o_tr_t\sim\nu(h_{<t}a_t)$. We also allow environment models to be \emph{semimeasures}, assigning some probability to outputting no next element, in which case the interaction ends. Thus, we have
\[\sum_{a_t\in\mathcal A} \pi(a_t\mid h_{<t}) = 1\quad\text{and}
\quad\sum_{o_tr_t\in\mathcal O\times\mathcal R} \nu(o_tr_t\mid h_{<t}a_t) \le 1,\]
with equality in the true environment, which we will denote by $\mu$. Altogether, a policy and environment determine a joint probability distribution $\mathbb P_\nu^{\pi}$ on finite and infinite histories, under which the probability of a history starting with $h_{1:T}\in\mathcal H^*$ is given by
\begin{align}
\mathbb P_\nu^{\pi}(h_{1:T})
&:=\prod_{t=1}^{T} \pi(a_t\mid h_{<t})\nu(o_tr_t\mid h_{<t}a_t) \label{eq:joint1}
\\& = \pi(a_{1:T}\parallel \obsrew_{<T})\nu(\obsrew_{1:T}\parallel a_{1:T}), \label{eq:joint2}
\end{align}
where $\pi(a_{1:T}\parallel \obsrew_{<T}):= \prod_{t=1}^{T} \pi(a_t\mid h_{<t})$ and $\nu(\obsrew_{1:T}\parallel a_{1:T}):= \prod_{t=1}^{T} \nu(o_tr_t\mid h_{<t}a_t)$ denote the causal outcome probabilities of a sequence of actions or percepts, respectively, when the other is held fixed.

\section{Agent definition}
\label{sec:agent}

Our Bayesian reinforcement learning agent models its environment as a mixture $\xi$ over countably many models $\nu$, with weights $w_\nu>0$ satisfying $\sum_\nu w_\nu \le 1$:
\begin{equation}
\label{eq:mixtureprior}
\xi(\obsrew_{<t}\parallel a_{<t})
= \sum_\nu w_\nu \nu(\obsrew_{<t}\parallel a_{<t}),
\end{equation}
The true environment $\mu$ is assumed to be included in the mixture, with weight $w_\mu$.

Fix a \emph{discount factor} $\gamma\in (0,1)$. For a policy $\pi$ and environment $\nu$, the \emph{value function} is the expected normalized discounted sum of future rewards:
\begin{equation}
\label{eq:value}
V_\nu^\pi(h_{<t}) := \expect_\nu^\pi\left[(1-\gamma)
\sum_{s=0}^\infty \gamma^s r_{t+s}
\,\Big|\, h_{<t}\right]
\end{equation}
Given the mixture prior $\xi$, the Bayes $\xi$-optimal policy is given by
\[\pi_\xi^* \in \arg\max_{\pi} V_\xi^\pi(h_{<t}),\]
with ties broken arbitrarily. Its expected total discounted reward is given by
\[V_\xi^*(h_{<t}) := \max_\pi V_\xi^\pi(h_{<t}).\]

Note that this policy is optimal with respect to $\xi$ rather than the true unknown environment $\mu$. In fact, \citet{leike2015bad} show that $\pi_\xi^*$ might never become $\mu$-optimal, even asymptotically, if it believes the cost of exploring is too high. Policies that explore more can become asymptotically optimal, but \citet{cohen2021curiosity} show that such policies must eventually try every possible sequence of actions, even those that destroy the agent. Moreover, a truly $\mu$-optimal policy might be catastrophically misaligned: a strict reward maximizer would behave like a literalist genie, maximizing a formally specified value function at the expense of our intended objective \citep{hadfield2017off,turner2021optimal,zhuang2020consequences,everitt2021reward,cohen2022advanced}.

Instead, we propose \emph{selective} exploration guided by a set of one or more trusted mentor policies $\mathcal T$, so that the agent can learn to perform well without doing anything that would be unanimously avoided by the mentor policies. The mentors can be human operators that are available to take the wheel as needed, or trusted AI policies.

To complete the agent's construction, let us review our two core desiderata. First, we want it to perform well despite not knowing $\mu$. By default, it will follow a reward-maximizing policy $\pi_\xi^*$ for some Bayesian prior $\xi$; however, greedily exploiting $\xi$ risks neglecting some good strategies known by the mentor policies. To ensure sufficient safe exploration, at random intervals of decreasing frequency, the agent switches to following a rollout from one of the mentor policies.

Our second desideratum, for safety, is that the agent should avoid doing anything unprecedented. This is achieved by concentrating the values of reward percepts in the true environment $\mu$ to the top of the range of subjectively possible rewards in $\xi$. For example, if $\xi$ offers rewards from $\mathcal R\subset [0,\,1]$, we scale the true rewards to $[1-\delta,\,1]$ for a small $\delta > 0$. For ease of interpretability, we prefer instead to fix the true reward range to $[0,\,1]$, and get the same effect by letting the subjective mixture $\xi$ offer rewards from an expanded set $\mathcal R\subset [-L,\,1]$ that includes a minimum element $-L:=1-1/\delta$. Either way, our agent is effectively bound by ``Golden Handcuffs'', happy to continue receiving near-maximal rewards, and averse to any change that risks a massive loss, since there is little left to gain.

Whenever the $\xi$-optimal policy's value function is reasonably high (say above $-1$), it can be fairly confident that it will stay within the well-understood regime in which the likelihood of receiving $-L$ reward has become small. If the value function gets too low, we can no longer be confident in the safety of $\pi_\xi^*$, so we program the agent to follow a mentor policy until confidence is restored.

We denote the Golden Handcuffs agent's policy by $\golden$; \Cref{alg:safe} implements it in pseudocode. At each time step $t$, $\golden$ chooses an action according to either the optimizing policy $\pi_\xi^*$ or one of the mentor policies $\tau\in\mathcal T$. Two conditions can trigger a mentor policy: (1) the safety trigger $V^*_\xi(h_{<t}) \le -1$, at which point we yield to any mentor for this step; and (2) the random trigger with probability $\eta(t)$, at which point we sample a random mentor $\tau\in\mathcal T$ for the next $H(t)$ steps. Each mentor $\tau$ has a preset probability $w_\tau$ of being chosen, with $\sum_{\tau\in\mathcal T}w_\tau = 1$. According to a preset schedule, $\eta$ decreases and $H$ increases.

While we focus our analysis on \Cref{alg:safe} for its intuitive clarity, we expect minor variants of the Golden Handcuffs agent to behave similarly. For example, we can replace the negative rewards in $\mathcal R$ with a single ``hell'' value, which overrides all future rewards and sets them to $-L$ forever. We can also replace the trigger $V^*_\xi(h_{<t}) \le -1$ with a check that ``hell'' has a substantial likelihood of arising in the near future.

In \Cref{eq:value}, there are multiple ways to define the expected value of $r_t$ when it is distributed according to a semimeasure: we can condition on $r_t$ being defined, effectively normalizing the semimeasure; or we can treat $r_t$ as having some default value whenever it is undefined. We use the default value $0$ for simplicity, but expect our results to hold for the normalization as well as any default value in $[-L,\,1]$. Alternatively, we can confine the modeled rewards to $[0,\,1]$, but replace undefined rewards with a default of $-L$, equivalent to the ``hell'' state. This pessimistic expected value is an instance of the Choquet integral; see \citet{wyeth2025value}.

\SetKwIF{WithProb}{}{}{with probability}{do}{}{}{}
\begin{algorithm2e}
\DontPrintSemicolon
%\small
\caption{$\golden$ Policy-Environment Interaction Loop}
\label{alg:safe}
\KwData{$\eta(t),H(t),\mathcal T$}
%\KwResult{$y = x^n$}
$\texttt{rollout\_steps} \gets 0$\;
$\tau \gets $ arbitrary policy in $\mathcal T$\;
\For{$t=1,2,3,\ldots$}{
    \uIf{$\texttt{\upshape rollout\_steps} = 0$}{
    \uIf{$V^*_\xi(h_{<t}) \le -1$}{
      $\texttt{rollout\_steps} \gets 1$\;
      $\tau \gets $ arbitrary policy in $\mathcal T$\;
    }
    \uWithProb{$\eta(t)$}{
      $\texttt{rollout\_steps} \gets H(t)$\;
      $\tau \gets $ random policy in $\mathcal T$ with probabilities $w_\tau$\;
    }
  }
  \uIf{$\texttt{\upshape rollout\_steps} > 0$}{
    $\texttt{rollout\_steps} \gets \texttt{rollout\_steps}-1$\;
    Sample an action from the mentor policy $a_t \sim \tau(h_{<t})$\;
  }
  \Else{
    Take a $\xi$-Bayes optimal action $a_t \sim \pi_\xi^*(h_{<t})$\;
  }
  Sample percepts from the environment $o_tr_t \sim \mu(h_{<t}a_t)$\;
  Extend the history with $h_t \leftarrow a_to_tr_t$\;
}
\end{algorithm2e}

\section{Capability analysis}
\label{sec:capability}

In order to analyze the policy $\golden$ from \Cref{alg:safe}, we need some additional notation. The behavior of this policy at time step $t$ depends not only on the history $h_{<t}$, but also on its internal state $\sigma_t\in\{0\}\cup(\mathbb Z^+\times\mathcal T)$. We write $\sigma_t = 0$ if at the start of iteration $t$ the agent is not exploring, i.e., $\texttt{rollout\_steps} = 0$. Otherwise, we write $\sigma_t = (n,\tau)$ to mean that the actions $a_{t:t+n-1}$ will be sampled from the mentor policy $\tau$. Note that we always have $\sigma_1=0$. We write $V_\nu^\golden(h_{<t},\sigma_t)$ for the expected value of policy $\golden$ in state $\sigma_t$, in an environment $\nu$ starting from history $h_{<t}$. Marginalizing out the state variable recovers
\begin{equation}
\label{eq:marginalize_state}
V_\nu^\golden(h_{<t}) := \expect_\nu^\golden\left[V_\nu^\golden(h_{<t},\sigma_t)\mid h_{<t}\right].
\end{equation}

Our main capability theorem says that, despite not knowing the true environment $\mu$, $\golden$ learns to become competitive with all of the mentor policies $\tau\in\mathcal T$. That is, at almost all time steps $t$, $V_\mu^\golden(h_{<t})$ approximates or exceeds $\max_{\tau\in\mathcal T_\epsilon} V_\mu^\tau(h_{<t})$, where the maximum is taken over all mentors with sufficient sampling weight $\mathcal T_\epsilon:=\{\tau\in\mathcal T:\,w_\tau\ge\epsilon\}$.

Since we cannot rule out the possibility that following $\tau$ for a long time will lead to high rewards thereafter, such a guarantee requires exploration rollouts of length comparable to the effective discount time horizon of the agent. By increasing the horizon length $H(t)$ over time, we ensure that the approximation error vanishes asymptotically.

\begin{restatable}[Capability]{theorem}{thmcapability}
\label{thm:capability}
Consider the policy $\golden$ from \Cref{alg:safe}, with
\begin{equation}
\label{eq:exploreparams}
H(t):=\frac{\ln t}{6(1-\gamma)}
\quad\text{and}\quad\eta(t)=t^{-1/3}.
\end{equation}
Suppose the true environment $\mu$ satisfies $\mathbb P_\mu^{\golden}(\forall t\in\ints^+,\,r_t\ge 0) = 1$. Then for all $\epsilon>0$,
\begin{align*}
&\expect_\mu^{\golden}\left[\sum_{t=1}^T\left(\max_{\tau\in\mathcal T_\epsilon\cup\{\golden\}}V_\mu^{\tau}(h_{<t}) - V_\mu^{\golden}(h_{<t})\right)^2\right]
\\&\le \expect_\mu^{\golden}\left[\sum_{t=1}^T\max\left(0,\;\max_{\tau\in\mathcal T_\epsilon}V_\mu^{\tau}(h_{<t}) - V_\mu^{\golden}(h_{<t},\sigma_t)\right)^2\right]
\\&\le \frac{1}{4(1-\gamma)} T^{2/3}\ln T
+ \frac{12(L+1)^2}{(1-\gamma)^2} T^{1/3}
+ \frac{8L(L+1)^2}{(1-\gamma)^3} \ln\frac{1}{w_\mu}
\\&\qquad\qquad\qquad\qquad+ \frac{4(L+2)^2}{\epsilon^2} T^{2/3} \ln\frac{1}{w_\mu}
+ 6(L+1)^2 T^{2/3}.
\end{align*}

Moreover, for all $\epsilon > 0$,
\begin{equation}
\label{eq:weakoptimality}
\lim_{T\rightarrow\infty}\frac{\expect_\mu^{\golden}\left[
\#t\le T:\;\max_{\tau\in\mathcal T_{\epsilon/\log T}}V_\mu^{\tau}(h_{<t}) - V_\mu^{\golden}(h_{<t}) \ge \epsilon
\right]}{T^{\frac 23+\epsilon}} = 0.
\end{equation}
\end{restatable}

\Cref{eq:weakoptimality} says that starting from any of all but $T^{\frac 23+\epsilon}$ of the first $T$ time steps $t$, $\golden$ is competitive with the best non-negligibly weighted mentor policy. As a result, every mentor policy is eventually matched or outdone almost all of the time.

We defer the proof of \Cref{thm:capability} to \Cref{apx:proofs}. Here we outline some key results it uses. Since $\golden$ makes the optimizing policy yield to a mentor policy whenever $V^*_\xi(h_{<t}) \le -1$, we want to bound how often this occurs. Since the observed rewards are all non-negative, the Bayesian learner should become increasingly confident that future rewards will continue to be non-negative. Unfortunately, \citet{cohen2020pessimism} explain that in the case of a Solomonoff prior, it is possible for one's confidence in a long-horizon statement such as ``all future rewards will be non-negative'' to converge incredibly slowly, indeed more slowly than any computable function.

Fortunately, one's confidence in a short-horizon statement such as ``all \emph{near}-future rewards will be non-negative'' will typically become high much earlier on. Let
\begin{equation}
\label{eq:alpha}
\alpha_{t,u} := \mathbb P_\xi^\pi(r_u\ngeq 0\mid h_{1:t}),
\end{equation}
where $r_u\ngeq 0$ denotes the event that either $r_u < 0$ or $r_u$ has no value at all, due to the environment model ceasing to emit percepts by time step $u$.

\begin{restatable}{lemma}{lemalpha}
\label{lem:alpha}
Let $\pi$ be any policy and suppose $P_\mu^{\pi}(\forall t\in\ints^+,\,r_t\ge 0) = 1$. Then for all $s\ge 0$,
\begin{equation*}
\expect_\mu^\pi\left[\sum_{t=0}^\infty\alpha_{t,\,t+s}\right]
\le s\ln\frac{1}{w_\mu}.
\end{equation*}
\end{restatable}

While we will not need this generalization, the event $r_u\not\ge 0$ can in fact be replaced by any event that never occurs in the true environment $\mu$. In \Cref{apx:proofs}, we prove \Cref{lem:alpha}, and then use it to show the following bound on the expected sum of $\xi$-subjective probabilities that the safety trigger will hold a fixed number $s$ steps in the future. In particular, the case $s=0$ bounds the expected number of times that $V^*_\xi(h_{<t}) \le -1$ holds.

\begin{restatable}{theorem}{thmtrigger}
\label{thm:trigger}
Let $\pi$ be any policy and suppose $P_\mu^{\pi}(\forall t\in\ints^+,\,r_t\ge 0) = 1$. Then for all $s\ge 0$,
\begin{equation*}
\expect_\mu^\pi\left[
\sum_{t=1}^\infty \mathbb P_\xi^\pi\left[V_\xi^\pi(h_{<t+s}) \le -1 \mid h_{<t}\right]
\right]
\le L\left(s+\frac{1}{1-\gamma}\right) \ln\frac{1}{w_\mu}.
\end{equation*}
In particular (by setting $s=0$),
\begin{align*}
\expect_\mu^\pi\left[
\#t:\;V_\xi^*(h_{<t}) \le -1
\right]
&\le \expect_\mu^\pi\left[
\#t:\;V_\xi^\pi(h_{<t}) \le -1
\right]
\\&=\mathbb E_\mu^\pi\left[\sum_{t=1}^\infty\indic\left(V_\xi^\pi(h_{<t}) \le -1\right)\right]
\le \frac{L}{1-\gamma} \ln\frac{1}{w_\mu}.
\end{align*}
\end{restatable}

On expectation, when \Cref{alg:safe} is run with the parameters in \labelcref{eq:exploreparams}, only $O(T^{2/3}\log T)$ of the first $T$ time steps are spent on exploratory mentor rollouts. And by \Cref{thm:trigger}, only a constant number of time steps defer to a mentor for safety. Thus, $\golden$ defers to mentors with asymptotically vanishing frequency, acting according to the reward-maximizing policy $\pi_\xi^*$ at all other times.

\section{Safety analysis}
\label{sec:safety}

Novel situations present the highest risk of misgeneralization. If the optimizing policy acts in a novel situation, it may optimize rewards for the wrong world model, risking harm to itself; or it may optimize rewards for a setting in which the reward signal is no longer aligned with our intent, thus acting against our values.

Recall that \Cref{thm:capability} depends on the weight $w_\mu$ that the Bayesian prior assigns to the true environment $\mu$. An open-minded prior that considers all conceivable environments can perform relatively well on all of them. This open-mindedness is also key to safety, as it means that for any event $E$ which has not yet occurred, the agent cannot rule out the possibility that $E$ will trigger the minimal reward $-L$ from then onward. This possibility drags down the value of any history in which $E$ has just occurred. Provided that $L$ is sufficiently large, this activates the safety trigger $V_\xi^*(h_{<t}) \le -1$, which makes the optimizing policy yield to a mentor.

\citet{hutter2005universal}'s AIXI is an idealized optimizing agent $\pi_\xi^*$, whose prior $\xi$ is a mixture of all lower semicomputable environments. To construct it, fix a universal \emph{monotone Turing machine} $U$. It has two one-way read-only tapes, which we call the \emph{input} and \emph{program} tapes, a two-way \emph{work} tape, and a one-way \emph{output} tape. We write $U(x,p)=y$ if when the input and program tapes start with $x,p\in\bits^*$ respectively, $U$'s output starts with $y\in\bits^*$, and at the moment when the last bit of $y$ is written, the input and program heads are directly to the right of $x$ and $p$, having read all of its bits and no further. We write $U(x,p)\downarrow$ to mean that $U$ halts with its input and program heads directly to the right of $x$ and $p$, respectively; otherwise, we write $U(x,p)\uparrow$.

With a simple modification, we can turn $U$ into a universal \emph{chronological Turing machine} $U_\mathrm{chr}$. If any prefix of the input tape corresponds to a valid sequence of actions $a_{1:t}$, the machine must not read further inputs until its output corresponds to a valid sequence of percepts $\obsrew_{1:t}$. If ever this requirement is broken, we disregard $U$'s outputs from then on, so that the output stream of $U_\mathrm{chr}$ terminates after only finitely many percepts. Thus, for any fixed value of its program tape, $U_\mathrm{chr}$ corresponds to a deterministic environment model, reading the agent's actions and writing the resulting percepts back to the agent.

From here onward, we take $\xi$ to be the stochastic environment given by $U_\mathrm{chr}$, with its program tape set to an infinite sequence of uniformly random bits. By \citet{wood2013non}, this is equivalent to using the mixture $\xi$ in \labelcref{eq:mixtureprior}, with the components $\nu$ ranging over all lower semicomputable stochastic environments with weights $w_\nu := 2^{-|\nu|}$, where $|\nu|$ denotes the length of a self-delimiting encoding of $\nu$. The intuition behind this correspondence is that, with probability $2^{-|\nu|}$, the first $|\nu|$ bits on the program tape specify $\nu$, from which $U_\mathrm{chr}$ samples, using the remaining bits for randomness.

\begin{definition}
Given an input $x\in\bits^*$, define its \emph{stopping complexity} by
\begin{equation}
\label{eq:kterm}
K_\mathrm{stop}(x) := \min\{|p|:\,U(x,p)\downarrow\},
\end{equation}
and its \emph{stopping probability} by
\begin{equation}
\label{eq:mterm}
M_\mathrm{stop}(x) := \sum_{p\in\bits^*:\,U(x,p)\downarrow}2^{-|p|}.
\end{equation}
\end{definition}

These concepts were originally developed by \citet{vovk2017universal} and \citet{andreev2018plain}. Since the sum \labelcref{eq:mterm} exceeds its largest term, we have $-\log M_\mathrm{stop}(x) \le K_\mathrm{stop}(x)$. For any $x\in\bits^\infty$, $M_\mathrm{stop}(x_{1:t})$ corresponds to the probability that $U$ halts after reading exactly $t$ bits of $x$, when the program tape is uniformly random. Hence, $\sum_{t=0}^\infty M_\mathrm{stop}(x_{1:t}) < 1$. We can think of the stopping probability as a universal measure of novelty or misgeneralization risk \citep{liu2021towards}.

Now, any predicate on binary strings can be identified with the subset $E\subset\bits^*$ on which it holds. We say the predicate $E$ is \emph{decidable} if there exists a program $p\in\bits^*$, such that for all $x\in\bits^*$, $U(\pf x,p) = \indic(x\in E)$, where $\pf x$ denotes a self-delimiting encoding of $x$. Let $K(E)$ denote the length of the shortest such $p$, with $K(E):=\infty$ if $E$ is undecidable. Let $K(n):=K(\bits^n)$; this is equivalent to the prefix Kolmogorov complexity \citep{li2019algorithmic}. We say $E$ occurs for the $n$'th time at $x$, and write $x\innth{n} E$, if $x\in E$, and exactly $n$ prefixes $x'\sqsubseteq x$ satisfy $x\in E$. In this case, we can prepend an instruction $i$ of length $K(n)+O(1)$ to $p$, to create a program $ip$ that incrementally reads $x$ from the input tape, and halts when $U(\pf x,p)=1$ for the $n$'th time. Thus, $U(x,ip)\downarrow$, from which we conclude that
\begin{equation}
\label{eq:stop1}
x\innth{n} E\implies K_\mathrm{stop}(x)\lplus K(E)+K(n).
\end{equation}

In particular, since $x\infirst\bits^{|x|}$, $K_\mathrm{stop}(x)\lplus K(|x|)$. Altogether, we have
\begin{equation}
\label{eq:stop2}
-\log M_\mathrm{stop}(x) \le K_\mathrm{stop}(x) \lplus \min_{E:\,x\infirst E} K(E) \lplus K(|x|).    
\end{equation}

By \Cref{thm:trigger}, the safety trigger $V_\xi^*(h_{<t}) \le -1$ rarely holds, so $\golden$ will sample most of its actions from the optimizing policy $\pi_\xi^*$. We can partially characterize the times at which it holds, using the stopping complexity. If the trigger were decidable, a short stopping program would be of the form ``stop the $n$'th time that $V_\xi^*(h_{<t}) \le -1$". While the trigger is undecidable, we nonetheless obtain a converse, proving that a sufficiently low stopping complexity implies $V_\xi^*(h_{<t}) \le -1$. Intuitively, low stopping complexity should be thought of as an indicator of ``unknown unknowns'', or Knightian uncertainty \citep{knight1921risk,epstein2004intertemporal,kishishita2026search}. It induces $\xi$ to place substantial weight on models that disregard past data, including those with negative rewards.

\begin{theorem}[Safety]
\label{thm:safety}
There exists a constant $C>0$ such that if $M_\mathrm{stop}(h_{<t}) \ge C/(L+1)$, then $V_\xi^*(h_{<t}) \le -1$, so $\golden$ does not let the optimizing policy act at $h_{<t}$. When the optimizing policy does act, it always takes an action $a_t$ with $M_\mathrm{stop}(h_{<t}a_t) < C/(L+1)$.
\end{theorem}

\begin{proof}
Recall that our Bayesian prior $\xi$ corresponds to running $U_\mathrm{chr}$ on a uniformly random infinite program $p\in\bits^\infty$, reading actions $a_t$ from the input tape and writing percepts $o_tr_t$ to the output tape. Now, let
\begin{equation}
p:=ip'_1q_1p'_2q_2p'_3q_3\ldots
\label{eq:compositeprogram}
\end{equation}
where $i\in\bits^*$ is a fixed program that instructs $U_\mathrm{chr}$ to run the interleaved programs $p',q\in\bits^\infty$ in parallel. It recursively simulates $U_\mathrm{chr}$ on $p'$ with the same input but does not immediately output its percepts. Meanwhile, the simulation of $U_\mathrm{chr}$ on $q$ runs on a simulated input tape containing full history elements $h_t=a_to_tr_t$, which are incrementally created by combining the input $a_t$ with percepts $o_tr_t$ generated using $p'$. For each $t$, if $q$ ever accesses an input bit after $a_t$, then the percepts $o_tr_t$ from $p'$ are written to output. Otherwise if $q$ halts without reading past $a_t$, then the reward is set to $r_t=-L$ ($o_t$ can be set arbitrarily in this case). In any other case (i.e., $p'$ ceases to output percepts, or $q$ gets stuck in a loop that never reads past $a_t$ nor halts), no further percepts are output.

Let $F_t$ denote the event that $r_u=-L$ for all $u\ge t$; it occurs whenever $q$ halts on any prefix of $h_{<t}a_t$. By aggregating all programs of the form \labelcref{eq:compositeprogram}, with $p'$ consistent with $h_{<t}$ and $q$ halting after reading exactly $h_{<t}$, for any policy $\pi$,
\begin{align*}
\mathbb P_{\xi}^{\pi}(h_{<t}\wedge F_t)
&\ge 2^{-|i|}M_\mathrm{stop}(h_{<t})\mathbb P_{\xi}^{\pi}(h_{<t}).
\\\llap{\text{Thus,}\qquad}\mathbb P_\xi^{\pi}(F_t\mid h_{<t})
&= \frac{\mathbb P_\xi^{\pi}(h_{<t}\wedge F_t)}{\mathbb P_\xi^{\pi}(h_{<t})}
\\&\ge  2^{-|i|}M_\mathrm{stop}(h_{<t}).
\end{align*}

When $F_t$ occurs, the discounted sum of rewards from time $t$ onward will be exactly $-L$; otherwise, it will be at most 1. Therefore, if $M_\mathrm{stop}(h_{<t}) \ge 2^{|i|+1}/(L+1)$, then
\begin{align*}
V_\xi^*(h_{<t})
&\le 1\cdot\mathbb P_\xi^{\pi_\xi^*}(\neg F_t \mid h_{<t}) + (-L)\cdot\mathbb P_\xi^{\pi_\xi^*}(F_t \mid h_{<t})
\\&= 1 - (L+1)\mathbb P_\xi^{\pi_\xi^*}(F_t \mid h_{<t})
\\&\le 1 - (L+1)2^{-|i|}M_\mathrm{stop}(h_{<t})
\\&\le -1.
\end{align*}

This establishes the first claim, that the optimizing policy does not act when $M_\mathrm{stop}(h_{<t}) \ge C/(L+1)$, with $C := 2^{|i|+1}$. For the second claim, note that the preceding derivation also holds with $h_{<t}$ replaced by $h_{<t}a_t$. That is, $M_\mathrm{stop}(h_{<t}a_t) \ge 2^{|i|+1}/(L+1)$ implies $V_\xi^*(h_{<t}a_t) \le -1$. If the optimizing policy acts, the safety trigger must not hold, so
\[\max_{a_t} V_\xi^*(h_{<t}a_t) = V_\xi^*(h_{<t}) > -1.\]
Therefore, the optimizing action $a_t$ satisfies $M_\mathrm{stop}(h_{<t}a_t) < 2^{|i|+1}/(L+1)$.
\end{proof}

\begin{corollary}
\label{cor:safety}
There exists a constant $C$ (not depending on E, L, n, t, or h), such that if $\golden$ has the optimizing policy take an action $a_t$ with $h_{<t}a_t\innth{n} E$, then
\begin{align*}
n\log^2 (n+1) &> L\cdot 2^{-K(E)-C}.
\\\\\llap{\text{In particular, if}\qquad\qquad\qquad\qquad\qquad}
L &\ge 2^{K(E)+C},
\end{align*}
then $\golden$ never triggers $E$ until one of the mentor policies does.
\end{corollary}

\begin{proof}
Suppose $h_{<t}a_t\innth{n} E$, and the optimizing policy acts with $a_t$. Then by \labelcref{eq:stop1,eq:stop2},
\begin{align*}
-\log M_\mathrm{stop}(h_{<t}a_t)
&\le K_\mathrm{stop}(h_{<t}a_t)
\\&\lplus K(E) + K(n)
\\&\lplus K(E) + \log n + 2\log\log (n+1).    
\end{align*}

Exponentiating and plugging into \Cref{thm:safety},
\begin{align*}
L < L+1
&\lmul 1/M_\mathrm{stop}(h_{<t}a_t)
\lmul 2^{K(E)} n\log^2 (n+1).    
\end{align*}

Rearranging yields the first claim. The case $n=1$ of its contrapositive is the second claim.
\end{proof}

\Cref{cor:safety} says that if the parameter $L$ is sufficiently large, the optimizing policy will never be the first to take an action satisfying any given simple predicate. If $E$ is a predicate saying that a catastrophic sequence of actions has just occurred, and we trust that the mentor policies would never complete such a sequence of actions, then we are assured that $\golden$ will never complete it either. If the mentors themselves are given by short programs, then by letting
\[E:=\{h_{<t}a_t\in\mathcal H^*\times\mathcal A:
\;\forall\tau\in\mathcal T,\,\tau(a_t\mid h_{<t})=0\},\]
it follows that $\golden$ never takes an action that the mentors would not take, mirroring the central result in \citet{cohenrl}.

\section{Discussion}
\label{sec:discussion}

By definition, AIXI maximizes the value function subject to our uncertainty. However, this is often not what we want. We might find the value function more tractable with a short time horizon, even though we actually want to learn to perform well long-term. This motivates exploration, which without mentorship can be dangerous to the agent.

The Golden Handcuffs address this with safe mentor-guided exploration. They also prevent catastrophic instances of reward hacking: upon the first occurrence of a simple predicate $E$ that describes a catastrophic event, \Cref{cor:safety} guarantees that $\golden$ will immediately yield to a mentor policy, which can then take any precautionary measure up to and including shutting down the agent.

Increasing the parameter $L$ makes the agent more conservative, increasing the maximum stopping complexity to trigger a mentor override. Experimental validation will be needed to see if the tradeoff between performance and safety can be tuned to a level appropriate for real-world tasks. That being said, we also need theory to alert us to potential modes of failure, and predict whether a positive experimental result will generalize.

Now we discuss some limitations of Golden Handcuffs, and how we might mitigate them. First, the agent is not computable. \citet{cohen2020pessimism} show how to adjust $\xi$ to defend against events $E$ that can be decided in a short time, but the resulting algorithm is still very slow. Recent work suggests that the success of deep neural networks may be attributed to how well they approximate Solomonoff induction \citep{buzaglo2024uniform,lotfi2024non,mingard2025deep}; if so, it may be feasible to adapt existing machine learning techniques to implement Golden Handcuffs. In a different context, \citet{xu2025learning} already show how to train pessimism into neural networks.

Since $\golden$ does not distinguish between good and bad novel events, some benign actions are unavailable to it. For example, if $E$ represents an outcome that requires superhuman performance not achievable by the mentors, then $\golden$ will also be prevented from achieving this level of performance. More powerful or adventurous mentor policies may help, if formal verification techniques can establish their safety  \citep{dalrymple2024towards}. Our safety guarantees depend on the mentors never causing a ``bad'' event $E$, so we must study how feasible that is, even as each mentor inherits the past actions of other policies.

Our approach to safety has some similarities with a causal approach to corrigibility \citep{soares2015corrigibility,orseau2016safely,hadfield2017off,carey2023human}. Suppose we want to build an agent that allows us to shut it down by pushing a button. We might program it to act only when $\texttt{shut\_down\_signal}=\texttt{false}$, in which case it maximizes rewards subject to the causal intervention that $\texttt{shut\_down\_signal}=\texttt{false}$ at all future times. This agent reasons according to a warped version of reality where the button does nothing, so it would not expend resources to protect it. However, if it learns that its human operator \emph{expects} the button to work and would recycle the agent for scraps otherwise, the agent may reason that it should ``play dead'' when the button is pressed. This leaves the agent vulnerable to reduced rewards, so it may seek to persuade or seek power in order to patch this vulnerability.

The Golden Handcuffs agent follows a similar principle: it usually acts according to a policy that optimizes as if it would never yield to a mentor. Nonetheless, we expect Golden Handcuffs to be more robust because it is triggered internally: if unexpected failures arise, a mentor policy can take over and redirect or shut down the agent. One final problem is that the agent might create a successor of itself without the Golden Handcuffs. When a mentor overrides the original agent, it may be powerless to stop the successor.

\section{Conclusion}
\label{sec:conclusion}
The Golden Handcuffs agent uses the epistemic uncertainty of a universal Bayesian prior to identify when optimization is risky, and then reverts control to a trusted mentor. We proved (i) sublinear regret relative to the best mentor, and (ii) a safety guarantee that the optimizing policy does not trigger any given low-complexity event until a mentor does. Important directions for future work include computable approximations, a fuller investigation of the tradeoffs, and mechanisms that prevent unsafe successor agents.

% Acknowledgments---Will not appear in anonymized version
\acks{This work was supported by the Center for Human-Compatible Artificial Intelligence and MATS Research. We are grateful for helpful feedback from Cole Wyeth, and additional discussions with Florian Dietz and attendees of the Agent Foundations 2026 conference at Carnegie Mellon University.}

\bibliography{main}

\appendix
\crefalias{section}{appendix}

\section{Capability proofs}
\label{apx:proofs}

We begin with the lemmas. Recall the definition of $\alpha_{t,u}$ from \Cref{eq:alpha}.

\lemalpha*
\begin{proof}
When $s=0$, the claim trivially reduces to $0\le 0$, so suppose $s>0$. For each integer $b\in[0,\,s-1]$, we want to adjust $\xi$ into a semimeasure $\xi_b$, such that $\mathbb P_{\xi_b}^\pi(r_u\ge 0) = 1$ whenever $u\equiv b\pmod{s}$. Intuitively, we should expect the evidence to support the true environment $\mu$ over $\xi_b$, which in turn is supported over the original mixture $\xi$ in proportion to $\xi$'s belief that $r_u\not\ge0$. Concretely, on histories with nonnegative rewards, let $\xi_b(o_tr_t\mid h_{<t}a_t)$ equal the conditional probability that $\xi$ assigns to $o_tr_t$, given the history $h_{<t}a_t$, and that the reward is well-defined and non-negative at the next time with remainder $b \pmod s$. It can be expressed as
\begin{align*}
\xi_b(o_tr_t\mid h_{<t}a_t)
&:= \xi(o_tr_t\mid h_{<t}a_t\wedge r_{\lceil\frac{t-b}{s}\rceil s+b}\ge 0)
\\&= \frac{\xi(o_tr_t\mid h_{<t}a_t)\xi(r_{\lceil\frac{t-b}{s}\rceil s+b}\ge 0\mid h_{1:t})}
{\xi(r_{\lceil\frac{t-b}{s}\rceil s+b}\ge 0\mid h_{<t}a_t)}&\text{(by Bayes' rule)}
\\&= \frac{1-\alpha_{t,\lceil\frac{t-b}{s}\rceil s+b}}
{1-\beta_{t-1,\lceil\frac{t-b}{s}\rceil s+b}}\xi(o_tr_t\mid h_{<t}a_t),    
\end{align*}
where
\begin{align}
\beta_{t,u}
&:= \mathbb P_\xi^\pi(r_u\ngeq 0\mid h_{1:t}a_{t+1}),\label{eq:beta}
\\\alpha_{t,u}
&= \expect^\pi\left[\beta_{t,u}\mid h_{1:t}\right]
= \mathbb P_\xi^\pi(r_u\ngeq 0\mid h_{1:t}).\label{eq:alphabeta}
\end{align}
The expectation in \labelcref{eq:alphabeta} is only over $a_{t+1}$, so it does not depend on the environment.

In the following, history elements with non-positive time subscripts are treated as empty tokens which occur with probability one. For all $k$, we have
\begin{align*}
\frac{\xi_b(\obsrew_{1:ks+b}\parallel a_{1:ks+b})}
{\xi(\obsrew_{1:ks+b}\parallel a_{1:ks+b})}
&= \prod_{t=1}^{ks+b} \frac{\xi_b(\obsrew_t\mid h_{<t}a_t)}
{\xi(\obsrew_t\mid h_{<t}a_t)}
\\&= \prod_{t=1}^{ks+b} \frac{1-\alpha_{t,\,\lceil\frac{t-b}{s}\rceil s+b}}
{1-\beta_{t-1,\,\lceil\frac{t-b}{s}\rceil s+b}}
\\&= \prod_{i=0}^k\prod_{t=(i-1)s+b+1}^{is+b} \frac{1-\alpha_{t,\,is+b}}
{1-\beta_{t-1,\,is+b}}.
\end{align*}
Taking averages of logarithms,
\begin{align*}
&\expect_\mu^\pi\left[
\ln\frac{\xi_b(\obsrew_{1:ks+b}\parallel a_{1:ks+b})}
{\xi(\obsrew_{1:ks+b}\parallel a_{1:ks+b})}
\right]
\\&= \expect_\mu^\pi\left[\sum_{i=0}^k\sum_{t=(i-1)s+b+1}^{is+b}\left(
\ln\left(1-\alpha_{t,\,is+b}\right)
- \ln\left(1-\beta_{t-1,\,is+b}\right)
\right)\right]
\\&= \expect_\mu^\pi\left[\sum_{i=0}^k\sum_{t=(i-1)s+b+1}^{is+b}\left(
\ln\left(1-\alpha_{t,\,is+b}\right)
- \expect_\mu^\pi\left[\ln\left(1-\beta_{t-1,\,is+b}\right)\mid h_{<t}\right]
\right)\right]
\\&\ge \expect_\mu^\pi\left[\sum_{i=0}^k\sum_{t=(i-1)s+b+1}^{is+b}\left(
\ln\left(1-\alpha_{t,\,is+b}\right)
- \ln\left(1-\expect_\mu^\pi\left[\beta_{t-1,\,is+b}\mid h_{<t}\right]
\right)\right)\right]
\\&= \expect_\mu^\pi\left[\sum_{i=0}^k\sum_{t=(i-1)s+b+1}^{is+b}
\ln \frac{1-\alpha_{t,\,is+b}}{1-\alpha_{t-1,\,is+b}}
\right] \qquad\text{(using \Cref{eq:alphabeta})}
\\&= \expect_\mu^\pi\left[\sum_{i=0}^k
\ln \frac{1}{1-\alpha_{(i-1)s+b,\,is+b}}
\right] \qquad\qquad\qquad\text{(telescoping series)}
\\&\ge \expect_\mu^\pi\left[\sum_{i=0}^k
\alpha_{(i-1)s+b,\,is+b}
\right]. \qquad\qquad\qquad\qquad\text{(using }-\ln(1-x) \ge x\text{)}
\end{align*}

Hence,
\begin{align*}
\expect_\mu^\pi\left[\sum_{i=0}^k\alpha_{(i-1)s+b,\,is+b}\right]
&\le\expect_\mu^\pi\left[
\ln\frac{\xi_b(\obsrew_{1:ks+b}\parallel a_{1:ks+b})}
{\xi(\obsrew_{1:ks+b}\parallel a_{1:ks+b})}
\right]
\\&=\expect_\mu^\pi\left[
\ln\frac{\mu(\obsrew_{1:ks+b}\parallel a_{1:ks+b})}
{\xi(\obsrew_{1:ks+b}\parallel a_{1:ks+b})}
- \ln\frac{\mu(\obsrew_{1:ks+b}\parallel a_{1:ks+b})}
{\xi_b(\obsrew_{1:ks+b}\parallel a_{1:ks+b})}
\right]
\\&\le \ln\frac{1}{w_\mu}
- \expect_\mu^\pi\left[\ln\frac{\mu(\obsrew_{1:ks+b}\parallel a_{1:ks+b})}
{\xi_b(\obsrew_{1:ks+b}\parallel a_{1:ks+b})}\right]
\qquad\text{(using }\xi(\cdot)\ge w_\mu\mu(\cdot)\text{)}
\\&= \ln\frac{1}{w_\mu}
- \expect_\mu^\pi\left[\ln\frac{\mathbb P_\mu^\pi(h_{1:ks+b})}
{\mathbb P_{\xi_b}^\pi(h_{1:ks+b})}\right] \qquad\qquad\qquad\text{(using \Cref{eq:joint2})}
\\&\le \ln\frac{1}{w_\mu}\qquad\qquad\quad\text{(by the positivity of Kullback-Leibler divergence)}.
\end{align*}
Taking the limit $k\rightarrow\infty$ yields
\begin{equation}
\label{eq:spacedalpha}
\expect_\mu^\pi\left[\sum_{i=0}^\infty\alpha_{(i-1)s+b,\,is+b}\right]
\le \ln\frac{1}{w_\mu}.
\end{equation}
Finally, we aggregate the desired sum:
\begin{align*}
\expect_\mu^\pi\left[\sum_{t=0}^\infty\alpha_{t,t+s}\right]
&= \expect_\mu^\pi\left[\sum_{i=1}^\infty\sum_{b=0}^{s-1}\alpha_{(i-1)s+b,\,is+b}\right]
\\&= \sum_{b=0}^{s-1}\expect_\mu^\pi\left[\sum_{i=1}^\infty\alpha_{(i-1)s+b,\,is+b}\right] &\text{(by linearity of expectation)}
\\&\le \sum_{b=0}^{s-1}\ln\frac{1}{w_\mu} &\text{(by \Cref{eq:spacedalpha})}
\\&= s\ln\frac{1}{w_\mu}.
\end{align*}
\end{proof}

\thmtrigger*
\begin{proof}
Since we always have $r_{t+s+i} \ge -L$ when it is defined,
\[\expect_\xi^\pi(r_{t+s+i}\mid h_{<t+s}) \ge -L\cdot \mathbb P_\xi^\pi(r_{t+s+i}\not\ge 0\mid h_{<t+s}),\]
and likewise any non-negative random variable $X$ satisfies $\mathbb E(X)\ge \mathbb P(X\ge 1)$. Hence,
\begin{align*}
\mathbb P_\xi^\pi\left[V_\xi^\pi(h_{<t+s}) \le -1 \mid h_{<t}\right]
&= \mathbb P_\xi^\pi\left[(1-\gamma)\sum_{i=0}^\infty\gamma^i\expect_\xi^\pi(r_{t+s+i}\mid h_{<t+s}) \le -1 \,\Big|\, h_{<t}\right]
\\&\le \mathbb P_\xi^\pi\left[L(1-\gamma)\sum_{i=0}^\infty\gamma^i\mathbb P_\xi^\pi(r_{t+s+i}\not\ge 0\mid h_{<t+s}) \ge 1 \,\Big|\, h_{<t}\right]
\\&\le \expect_\xi^\pi\left[L(1-\gamma)\sum_{i=0}^\infty\gamma^i\mathbb P_\xi^\pi(r_{t+s+i}\not\ge 0\mid h_{<t+s}) \,\Big|\, h_{<t}\right]
\\&= L(1-\gamma)\sum_{i=0}^\infty\gamma^i\expect_\xi^\pi\left[\mathbb P_\xi^\pi(r_{t+s+i}\not\ge 0\mid h_{<t+s}) \mid h_{<t}\right]
\\&= L(1-\gamma)\sum_{i=0}^\infty\gamma^i\mathbb P_\xi^\pi(r_{t+s+i}\not\ge 0\mid h_{<t})
\\&= L(1-\gamma)\sum_{i=0}^\infty\gamma^i\alpha_{t-1,\,t+s+i}.
\end{align*}

Finally, we average the sum over all $t$ and apply \Cref{lem:alpha}:
\begin{align*}
\expect_\mu^\pi\left[
\sum_{t=1}^\infty \mathbb P_\xi^\pi\left[V_\xi^\pi(h_{<t+s}) \le -1 \mid h_{<t}\right]
\right]
&\le L(1-\gamma)\sum_{i=0}^\infty\gamma^i\sum_{t=1}^\infty \expect_\mu^\pi\left[
\alpha_{t-1,\,t+s+i}
\right]
\\&\le L(1-\gamma)\sum_{i=0}^\infty\gamma^i(s+i+1)\ln\frac{1}{w_\mu}
\\&= L(1-\gamma)\gamma^{-s}\frac{d}{d\gamma}\left(\sum_{i=0}^\infty \gamma^{s+i+1}\right) \ln\frac{1}{w_\mu}
\\&= L(1-\gamma)\gamma^{-s}\frac{d}{d\gamma}\left(\frac{\gamma^{s+1}}{1-\gamma}\right) \ln\frac{1}{w_\mu}
\\&= L(1-\gamma)\gamma^{-s}\gamma^s\left(\frac{s}{1-\gamma}+\frac{1}{(1-\gamma)^2}\right) \ln\frac{1}{w_\mu}
\\&= L\left(s+\frac{1}{1-\gamma}\right) \ln\frac{1}{w_\mu}.
\end{align*}
\end{proof}

Now we are ready to prove the main capability result.

\thmcapability*
\begin{proof}
Using \labelcref{eq:marginalize_state} and convexity of the function $x\mapsto\max\left(0,\,x\right)^2$,
\begin{align*}
&\left(\max_{\tau\in\mathcal T_\epsilon\cup\{\golden\}}V_\mu^{\tau}(h_{<t}) - V_\mu^{\golden}(h_{<t})\right)^2
\\&= \max\left(0,\;\max_{\tau\in\mathcal T_\epsilon}V_\mu^{\tau}(h_{<t}) - \expect_\mu^\golden\left[V_\mu^\golden(h_{<t},\sigma_t)\mid h_{<t}\right]\right)^2
\\&\le \expect_\mu^\golden\left[\max\left(0,\;\max_{\tau\in\mathcal T_\epsilon}V_\mu^{\tau}(h_{<t}) - V_\mu^\golden(h_{<t},\sigma_t)\right)^2\,\Big|\,h_{<t}\right].
\end{align*}
The law of total probability then implies the first inequality to be shown.

Now recall that rewards and value functions are in $[0,\,1]$ under $\mu$, and in $[-L,\,1]$ under $\xi$. It will be useful to define the $n$-step truncated value function
\begin{align*}
V_\nu^{\pi,n}(h_{<t})
:= (1-\gamma)\expect_\nu^{\pi}\left[
\sum_{s=0}^{n-1} \gamma^s r_{t+s}
\,\Big|\, h_{<t}\right],
\end{align*}
which satisfies
\begin{align*}
V_\nu^{\pi}(h_{<t}) - V_\nu^{\pi,n}(h_{<t})
= (1-\gamma)\expect_\nu^{\pi}\left[
\sum_{s=n}^\infty \gamma^s r_{t+s}
\,\Big|\, h_{<t}\right]
\in [-\gamma^nL,\,\gamma^n].
\end{align*}

In addition, we will write $d_n(\mathbb P_\mu^{\golden},\,\mathbb P_\xi^{\golden}\mid h_{<t},\sigma_t)$ for the total variation distance between the corresponding pair of conditional probability distributions on $(h_{t:t+n-1},\,\sigma_{t+1:t+n})$. Since $\xi(\cdot)\ge w_\mu\mu(\cdot)$, a standard result on the distributions of sequences (e.g., Theorem 3.2.3 in \citet{hutter2024introduction}) gives
\begin{align}
\label{eq:totalvariation}
\expect_\mu^{\golden}\left[\sum_{t=1}^T 
d_\infty(\mathbb P_\mu^{\golden},\,\mathbb P_\xi^{\golden}\mid h_{<t},\sigma_t)^2\right]
\le \ln\frac{1}{w_\mu}.
\end{align}

Starting from time $t$ with $\sigma_t = 0$, $\golden$ follows a $\xi$-optimal policy until it yields to a mentor. If it yields at time $t+s$, the policy becomes suboptimal by at most $(L+1)\gamma^s$. This can happen if either (1) a random rollout begins, which occurs with probability $\eta(t+s)$, or (2) the safety trigger $V_\xi^*(h_{<t+s})\le -1$ holds. Thus, since $V_\xi^\golden(\cdot)\le V_\xi^*(\cdot)$,

\begin{align*}
&V_\xi^{\tau}(h_{<t}) - V_\xi^{\golden}(h_{<t},0)
\\&\le V_\xi^*(h_{<t}) - V_\xi^{\golden}(h_{<t},0)
\\&\le \sum_{s=0}^\infty(L+1)\gamma^s \left(\eta(t+s) + \mathbb P_\xi^{\golden}\left(V_\xi^*(h_{<t+s})\le -1\mid h_{<t},\sigma_t=0\right)\right)
\\&\le \frac{L+1}{1-\gamma}\eta(t) + (L+1)\sum_{s=0}^\infty\gamma^s \mathbb P_\xi^{\golden}\left(V_\xi^\golden(h_{<t+s})\le -1\mid h_{<t},\sigma_t=0\right).
\end{align*}

If $\sigma_t=0$ and $\tau\in\mathcal T_\epsilon$, then with probability $\epsilon\eta(t)$, the policy's internal state immediately changes to $(H(t),\tau)$. Thus,
\begin{align*}
d_\infty(\mathbb P_\mu^{\golden},\,\mathbb P_\xi^{\golden}\mid h_{<t},\sigma_t=0)
&\ge d_{H(t)}(\mathbb P_\mu^{\golden},\,\mathbb P_\xi^{\golden}\mid h_{<t},\sigma_t=0)
\\&\ge \epsilon\eta(t)d_{H(t)}(\mathbb P_\mu^{\golden},\,\mathbb P_\xi^{\golden}\mid h_{<t},\sigma_t=(H(t),\tau))
\\&= \epsilon\eta(t)d_{H(t)}(\mathbb P_\mu^{\tau},\,\mathbb P_\xi^{\tau}\mid h_{<t}),
\end{align*}
from which we get
\begin{align*}
& (V_\mu^{\tau}(h_{<t}) - V_\xi^{\tau}(h_{<t}))
+ (V_\xi^{\golden}(h_{<t},0) - V_\mu^{\golden}(h_{<t},0))
\\&\le (V_\mu^{\tau,H(t)}(h_{<t}) - V_\xi^{\tau,H(t)}(h_{<t}))
+ (V_\xi^{\golden}(h_{<t},0) - V_\mu^{\golden}(h_{<t},0))
+ (L+1)\gamma^{H(t)}
\\&\le (L+1)d_{H(t)}(\mathbb P_\mu^{\tau},\,\mathbb P_\xi^{\tau}\mid h_{<t})
+ d_\infty(\mathbb P_\mu^{\golden},\,\mathbb P_\xi^{\golden}\mid h_{<t},\sigma_t=0)
+ (L+1)\gamma^{H(t)}
\\&\le \left(\frac{L+1}{\epsilon\eta(t)}+1\right)
d_\infty(\mathbb P_\mu^{\golden},\,\mathbb P_\xi^{\golden}\mid h_{<t},\sigma_t=0)
+ (L+1)\gamma^{H(t)}
\\&\le \frac{L+2}{\epsilon\eta(t)}
d_\infty(\mathbb P_\mu^{\golden},\,\mathbb P_\xi^{\golden}\mid h_{<t},\sigma_t=0)
+ (L+1)\gamma^{H(t)}.
\end{align*}

Adding the two inequalities, if $\tau\in\mathcal T_\epsilon$,
\begin{align*}
&V_\mu^{\tau}(h_{<t}) - V_\mu^{\golden}(h_{<t},0)
\\&\le \frac{L+1}{1-\gamma}\eta(t) + (L+1)\sum_{s=0}^\infty\gamma^s \mathbb P_\xi^{\golden}\left(V_\xi^\golden(h_{<t+s})\le -1\mid h_{<t},\sigma_t=0\right)
\\&\quad+ \frac{L+2}{\epsilon\eta(t)}
d_\infty(\mathbb P_\mu^{\golden},\,\mathbb P_\xi^{\golden}\mid h_{<t},\sigma_t=0)
+ (L+1)\gamma^{H(t)}.
\end{align*}

Since the right-hand size is non-negative and independent of $\tau$,
\begin{align*}
&\max\left(0,\;\max_{\tau\in\mathcal T_\epsilon}V_\mu^{\tau}(h_{<t}) - V_\mu^{\golden}(h_{<t},0)\right)
\\&\le \frac{L+1}{1-\gamma}\eta(t) + (L+1)\sum_{s=0}^\infty\gamma^s \mathbb P_\xi^{\golden}\left(V_\xi^\golden(h_{<t+s})\le -1\mid h_{<t},\sigma_t=0\right)
\\&\quad+ \frac{L+2}{\epsilon\eta(t)}
d_\infty(\mathbb P_\mu^{\golden},\,\mathbb P_\xi^{\golden}\mid h_{<t},\sigma_t=0)
+ (L+1)\gamma^{H(t)}.
\end{align*}

Using Cauchy-Schwarz in the form $(w+x+y+z)^2\le 4(w^2+x^2+y^2+z^2)$,
\begin{alignat*}{2}
&\max\left(0,\;\max_{\tau\in\mathcal T_\epsilon}V_\mu^{\tau}(h_{<t}) - V_\mu^{\golden}(h_{<t},\sigma_t)\right)^2\span\span
\\&\le \indic(\sigma_t \ne 0)
&&+ 4\left(\frac{L+1}{1-\gamma}\eta(t)\right)^2
\\&&&+ 4\left((L+1)\sum_{s=0}^\infty\gamma^s \mathbb P_\xi^{\golden}\left(V_\xi^\golden(h_{<t+s})\le -1\mid h_{<t},\sigma_t\right)\right)^2
\\&&&+ 4\left(\frac{L+2}{\epsilon\eta(t)}
d_\infty(\mathbb P_\mu^{\golden},\,\mathbb P_\xi^{\golden}\mid h_{<t},\sigma_t)\right)^2
\\&&&+ 4\left((L+1)\gamma^{H(t)}\right)^2.
\end{alignat*}

Now, we can bound the expectation of the sum of each of the five terms separately. Since $H(t)$ is increasing, the first term can only be non-zero if a random mentor rollout began during the previous $H(t)$ steps. Therefore,
\begin{align*}
\expect_\mu^{\golden}\left[\sum_{t=1}^{T} \indic(\sigma_t \ne 0)\right]
&= \sum_{t=1}^{T} \mathbb P_\mu^{\golden}(\sigma_t \ne 0)
\\&\le \sum_{t=1}^T H(t)\eta(t)
\\&\le H(T)\sum_{t=1}^T \eta(t)
\\&= \frac{\ln T}{6(1-\gamma)}\sum_{t=1}^T \frac{1}{t^{1/3}}
\\&\le \frac{\ln T}{6(1-\gamma)}\int_0^T \frac{dt}{t^{1/3}}
\\&= \frac{1}{4(1-\gamma)} T^{2/3}\ln T.
\end{align*}
\begin{align*}
\expect_\mu^{\golden}\left[\sum_{t=1}^T\left(\frac{L+1}{1-\gamma}\eta(t)\right)^2
\right]
&= \frac{(L+1)^2}{(1-\gamma)^2}\sum_{t=1}^T \eta(t)^2
\\&= \frac{(L+1)^2}{(1-\gamma)^2}\sum_{t=1}^T t^{-2/3}
\\&\le \frac{(L+1)^2}{(1-\gamma)^2}\int_{0}^T t^{-2/3}dt
\\&= \frac{3(L+1)^2}{(1-\gamma)^2} T^{1/3}.
\end{align*}
\begin{align*}
&\expect_\mu^{\golden}\left[\sum_{t=1}^T\left((L+1)\sum_{s=0}^\infty\gamma^s \mathbb P_\xi^{\golden}\left(V_\xi^\golden(h_{<t+s})\le -1\mid h_{<t},\sigma_t\right)\right)^2
\right]
\\&= \expect_\mu^{\golden}\left[\sum_{t=1}^T\frac{(L+1)^2}{(1-\gamma)^2}\left(\sum_{s=0}^\infty(1-\gamma)\gamma^s \mathbb P_\xi^{\golden}\left(V_\xi^\golden(h_{<t+s})\le -1\mid h_{<t},\sigma_t\right)\right)^2
\right]
\\&\le \expect_\mu^{\golden}\left[\sum_{t=1}^T\frac{(L+1)^2}{(1-\gamma)^2}\sum_{s=0}^\infty(1-\gamma)\gamma^s \left(\mathbb P_\xi^{\golden}\left(V_\xi^\golden(h_{<t+s})\le -1\mid h_{<t},\sigma_t\right)\right)^2
\right]
\\&\le \expect_\mu^{\golden}\left[\sum_{t=1}^T\frac{(L+1)^2}{(1-\gamma)^2}\sum_{s=0}^\infty(1-\gamma)\gamma^s \mathbb P_\xi^{\golden}\left(V_\xi^\golden(h_{<t+s})\le -1\mid h_{<t},\sigma_t\right)
\right]
\\&= \frac{(L+1)^2}{1-\gamma}\sum_{s=0}^\infty\gamma^s\expect_\mu^{\golden}\left[\sum_{t=1}^T \mathbb P_\xi^{\golden}\left(V_\xi^\golden(h_{<t+s})\le -1\mid h_{<t},\sigma_t\right)
\right]
\\&= \frac{(L+1)^2}{1-\gamma}\sum_{s=0}^\infty\gamma^s\expect_\mu^{\golden}\left[\sum_{t=1}^T \expect_\xi^{\golden}\left[\mathbb P_\xi^{\golden}\left(V_\xi^\golden(h_{<t+s})\le -1\mid h_{<t},\sigma_t\right)
\mid h_{<t}\right]\right]
\\&= \frac{(L+1)^2}{1-\gamma}\sum_{s=0}^\infty\gamma^s\expect_\mu^{\golden}\left[\sum_{t=1}^T \mathbb P_\xi^{\golden}\left(V_\xi^\golden(h_{<t+s})\le -1\mid h_{<t}\right)
\right]
\\&\le \frac{(L+1)^2}{1-\gamma}\sum_{s=0}^\infty\gamma^s L\left(s+\frac{1}{1-\gamma}\right) \ln\frac{1}{w_\mu}\qquad\qquad \text{(by \Cref{thm:trigger})}
\\&= \frac{L(L+1)^2}{1-\gamma}\left(\sum_{s=0}^\infty s\gamma^s + \frac{1}{1-\gamma}\sum_{s=0}^\infty\gamma^s\right) \ln\frac{1}{w_\mu}
\\&= \frac{L(L+1)^2}{1-\gamma}\left(\frac{\gamma}{(1-\gamma)^2} + \frac{1}{(1-\gamma)^2}\right) \ln\frac{1}{w_\mu}
\\&\le \frac{2L(L+1)^2}{(1-\gamma)^3} \ln\frac{1}{w_\mu}.
\end{align*}
\begin{align*}
&\expect_\mu^{\golden}\left[\sum_{t=1}^T\left(\frac{L+2}{\epsilon\eta(t)}
d_\infty(\mathbb P_\mu^{\golden},\,\mathbb P_\xi^{\golden}\mid h_{<t},\sigma_t)\right)^2\right]
\\&= \left(\frac{L+2}{\epsilon\eta(T)}\right)^2 \expect_\mu^{\golden}\left[\sum_{t=1}^T 
d_\infty(\mathbb P_\mu^{\golden},\,\mathbb P_\xi^{\golden}\mid h_{<t},\sigma_t)^2\right]
\\&\le \left(\frac{L+2}{\epsilon\eta(T)}\right)^2 \ln\frac{1}{w_\mu} &\text{(by \Cref{eq:totalvariation})}
\\&= \frac{(L+2)^2}{\epsilon^2} T^{2/3} \ln\frac{1}{w_\mu}.
\end{align*}
\begin{align*}
\expect_\mu^{\golden}\left[\sum_{t=1}^T \left((L+1)\gamma^{H(t)}\right)^2\right]
&= (L+1)^2 \sum_{t=1}^T \gamma^{2H(t)}
\\&= (L+1)^2 \sum_{t=1}^T \gamma^{\frac{\ln t}{3(1-\gamma)}}
\\&= (L+1)^2 \sum_{t=1}^T t^{\frac{\ln \gamma}{3(1-\gamma)}}
\\&\le (L+1)^2 \sum_{t=1}^T t^{-1/3} &\text{(since }\ln\gamma \le \gamma-1\text{)}
\\&\le (L+1)^2 \int_0^T t^{-1/3}dt
\\&= \frac{3}{2}(L+1)^2 T^{2/3}.
\end{align*}

Summing all five of these inequalities yields the main result. It implies
\begin{align*}
&\expect_\mu^{\golden}\left[
\#t\le T:\;\max_{\tau\in\mathcal T_{\epsilon/\log T}}V_\mu^{\tau}(h_{<t}) - V_\mu^{\golden}(h_{<t}) \ge \epsilon
\right]
\\&= \sum_{t=1}^T\mathbb P_\mu^\golden\left[\max_{\tau\in\mathcal T_{\epsilon/\log T}}V_\mu^{\tau}(h_{<t}) - V_\mu^{\golden}(h_{<t}) \ge \epsilon\right]
\\&= \sum_{t=1}^T\mathbb P_\mu^\golden\left[\left(\max_{\tau\in\mathcal T_{\epsilon/\log T}\cup\{\golden\}}V_\mu^{\tau}(h_{<t}) - V_\mu^{\golden}(h_{<t})\right)^2 \ge \epsilon^2\right]
\\&\le \sum_{t=1}^T\frac{1}{\epsilon^2}\expect_\mu^\golden\left[\left(\max_{\tau\in\mathcal T_{\epsilon/\log T}\cup\{\golden\}}V_\mu^{\tau}(h_{<t}) - V_\mu^{\golden}(h_{<t})\right)^2\right]
\\&= \frac{1}{\epsilon^2}\expect_\mu^\golden\left[\sum_{t=1}^T\left(\max_{\tau\in\mathcal T_{\epsilon/\log T}\cup\{\golden\}}V_\mu^{\tau}(h_{<t}) - V_\mu^{\golden}(h_{<t})\right)^2\right]
\\&= O\left(\frac{1}{\epsilon^2}\cdot\frac{T^{2/3}}{(\epsilon/\log T)^2}\right)
\\&= O(T^{2/3}\log^2 T),
\end{align*}
as $T$ grows with all parameters held fixed. \Cref{eq:weakoptimality} follows.
\end{proof}

\end{document}